\documentclass{article}
\usepackage[T1]{fontenc}     
\usepackage{hyperref}       
\usepackage{url}             
\usepackage{amsfonts}        
\usepackage{comment}
\usepackage{fancybox,framed}
\usepackage[margin=1.25in]{geometry}

\usepackage{bbding,pifont,amsmath, mathtools, cancel}

\usepackage{amsmath,amssymb,mathrsfs,amsthm}
\usepackage{times}
\usepackage{algorithm}
\usepackage[noend]{algpseudocode}

\usepackage{xcolor}
\definecolor{darkblue}{RGB}{0,0,128}
\definecolor{darkred}{RGB}{128,0,0}
\definecolor{darkgreen}{RGB}{0,128,0}

\usepackage{tikz}
\usepackage{graphicx}

\usepackage[export]{adjustbox}
\usepackage{multirow}
\usepackage{rotating}
\usepackage{booktabs}
\usepackage{subcaption}
\usepackage[utf8]{inputenc}

\renewcommand{\a}{\mathbf{a}}

\newcommand{\x}{\mathbf{x}}
\newcommand{\y}{\mathbf{y}}

\newcommand{\B}{\mathcal B}
\newcommand{\C}{\mathcal C}
\renewcommand{\O}{\mathcal O}
\newcommand{\R}{\mathbb R}
\renewcommand{\S}{\mathcal U}
\newcommand{\T}{\mathcal S}
\newcommand{\X}{\mathcal X}
\newcommand{\Y}{\mathcal Y}

\newcommand{\I}{\mathcal I}

\DeclareMathOperator*{\argmin}{\arg\!\min}
\DeclareMathOperator*{\argmax}{\arg\!\max}

\newcommand{\N}{\mathbb N}

\newcommand{\F}{\mathcal F}
\newcommand{\G}{\mathcal G}
\newcommand{\M}{\mathcal M}

\newcommand{\as}{\text{~a.s.}}
\newcommand{\dd}{\delta}

\def\polylog{\operatorname{polylog}}
\def\algo{\operatorname{Alg}}

\newtheorem{theorem}{Theorem} 
\newtheorem{definition}{Definition} 
 
\newtheorem{lemma}{Lemma} 
\newtheorem{proposition}{Proposition} 
\newtheorem{remark}{Remark} 
\newtheorem*{remark*}{Remark}
\newtheorem*{theorem*}{ \cite[Theorem~1]{khaleghi:12mce}}
\newtheorem*{lemma*}{\cite[Lemma~2]{khaleghi14}}

\title{Clustering piecewise stationary processes}
\author{Azadeh Khaleghi \\ \href{mailto:a.khaleghi@lancaster.ac.uk}{a.khaleghi@lancaster.ac.uk}\and Daniil Ryabko\\ \href{mailto:daniil@ryabko.net}{daniil@ryabko.net}}
\date{}

\begin{document}

\maketitle
\begin{abstract}
The problem of time-series clustering is considered in the case where each data-point is a sample generated by a piecewise stationary ergodic process. Stationary processes are perhaps the most general class of processes considered in non-parametric statistics and allow for arbitrary long-range dependence between variables. Piecewise stationary processes studied here for the first time in the context of clustering, relax the last remaining assumption in this model: stationarity. A natural formulation is proposed for this problem and a notion of consistency is introduced which requires the samples to be placed in the same cluster if and only if the piecewise stationary distributions that generate them have the same set of stationary distributions. Simple, computationally efficient algorithms are proposed and are shown to be consistent without any additional assumptions beyond piecewise stationarity.
\end{abstract}
\section{Introduction}
Clustering, in an informal sense, involves breaking a dataset into possibly disjoint subsets called clusters where the elements within the same cluster are somehow more similar to each other than to those in other clusters. This task, which is often a first step in data-analysis, is meant to help with the initial steps to making sense of the data that typically have complex structures and represent some unknown underlying phenomena to be inferred. Given the nature of the problem, it is desirable to make as little assumptions as possible about the underlying mechanisms generating the data. 
Moreover, the minimal  assumptions made should ideally be qualitative and easily verifiable from an application's perspective. 

In this paper we consider a subclass of the clustering problem where each data-point is a time series. Indeed, such sequential data are ubiquitous in modern applications involving, for example, user-behaviour, social networks,  as well as financial or biological data, where the observations are sequential by nature, and/or are collected over time. The common features in these real-world datasets are the absence of precise models as well as an abundance of data. 
From a mathematical perspective, a learning problem involving sequential data can be formulated as follows.
Given sequences of the form $Y_1,...,Y_n$ the aim is to make inference about the stochastic mechanisms that generate them. This task is typically done under the assumption that $Y_i$ are independently and identically distributed (i.i.d), or that their distribution belongs to a specific model class. 
However, such assumptions clearly undermine the possibly complex nature of data which may possess long-range 
dependencies. 

To address this gap, one approach is to assume that the process distributions are stationary without requiring any conditions to hold on their memory. That is, $Y_i$ need not be independent nor do they need to possess any finite-memory or mixing properties. 
This allows for arbitrary long-range dependencies between the observations. Moreover, thanks to Birkhoff's ergodic theorem, under this assumption alone, the frequency of each event almost surely converges to its underlying probability, even though there is no guarantee on the speed of this convergence. Intuitively, non-zero probability events that happen once, happen infinitely often, and their frequencies are meaningful. This characteristic is already enough to make some statistical inference.
In particular, in \cite{Ryabko:10clust} it was suggested to cluster time-series samples based on the distribution that generates them, putting together those and only those samples whose distribution is the same. Making use of the fact that in this setting the target clustering has the so-called strict separation property (e.g., \cite{Balcan:08}), it was shown that asymptotically consistent clustering is achievable under the assumption of stationarity alone; see also \cite{khaleghi2016consistent} and \cite{Ryabko:19c}.

While a rather weak assumption, stationarity often breaks in applications: events can happen from which there is no coming back. Simple real-world examples concerning user-behaviour may include such events as changing a job, changing a mobile phone, or having a child. Events that were typical in the past stop happening and events of a new kind start occurring. In these cases, the time index itself bares information. It may still make sense to measure frequencies of events in-between the changes, even though we may not be able to tell from the data alone, when the change happened. Moreover, the changes themselves may be described as abrupt, or, more generally, such that the transition stages are negligible at least asymptotically. 
Thus, the model that we suggest for the data here is based on piecewise stationary processes. Each time series that is given can be broken into segments, such that within each segment the distribution of the data is stationary. The segments' boundaries are arbitrary unknown, as are their corresponding distributions. Moreover, no assumptions beyond stationarity are made on the distributions of the stationary segments. In particular, the data within each segment are not assumed to be independent or possess any finite-memory or mixing properties.

We propose to identify a piecewise stationary distribution with a ``bag of distributions'' corresponding to the stationary segments. Thus, the locations where the change of distribution occurs, as well as the order of the stationary distributions, are disregarded. We consider a pair of piecewise stationary distributions equivalent if the set of stationary distributions of their segments coincide.
Thus, {\em the clustering objective is to  put into the same cluster those and only those time-series samples whose distributions  are equivalent in this sense.} An algorithm to achieve this objective is said to be consistent. 
Our {\em main result} is an algorithm that, as we show, is asymptotically consistent under the only assumption that each time series is stationary ergodic. The algorithm needs to know the correct number of clusters, which, as is already the case for stationary time series, is a necessary assumption in this setting. In addition, the length of each segment is required to be linear in the sample size, and a lower-bound on the length of the segments is assumed to be provided. These conditions formalize the intuition that transition periods are negligible with respect to the segments, and each segment is sufficiently long to allow for making inference. The consistency result is obtained using a novel distance between equivalence classes of piecewise stationary distributions. This distance is based on minimax distances between the distributions that generate the segments. The main technical argument establishes that this distance can be estimated consistently based on samples. 
The proposed algorithm extends those on clustering stationary time series \cite{Ryabko:10clust,khaleghi2016consistent}. It uses as sub-routines the algorithms for change-point analysis developed in \cite{khaleghi:12mce,khaleghi14}. Note, however, that the definition of the clustering objective  is more permissive here: the distributions of the time series within each cluster may be different, only the set of distributions of the segments has to be the same. 

The results provided in this paper are theoretical, and their main appeal is in their generality. While the proposed clustering algorithm is shown to be computationally feasible, its detailed experimental investigation is left as future work.
It should be noted that the literature on the related topic of change-point analysis is vast, but is mostly concerned with independent or mixing data, and also restricts the nature of the changes to single-dimensional marginals. The exceptions are the works cited above and references therein. 
To the best of our knowledge, there are no prior attempts to consider piecewise stationary distributions in the context of clustering. 
Outside of change-point analysis, related problems have been considered in the context of prediction, albeit with much more restrictive assumptions on the distributions between the changes. Thus, \cite{Willems:96} considers the prediction problem with arbitrary change points i.i.d. segments. See also~\cite{Gyorgy:12} and references therein. 

\section{Preliminaries}
We sometimes use the abbreviation $u..v$ for  
 $\{u,\dots,v\},~u\leq v \in \N$, and let
$|\cdot|$ denote the cardinality of a set or the number of elements in a sequence; distinction will be apparent from the context. 
Let $(\X, \B_{\X})$ be a measurable space; in this work we consider $\X=\R$, leaving extensions to more general spaces as future work. 
Denote by 
$\Delta_{u,v}:=\{[\frac{i_1}{2^{v}}, \frac{i_1+1}{2^{v}})\times \dots \times [\frac{i_u}{2^{v}}, \frac{i_u+1}{2^{v}}): i_j \in 0..2^{v}-1,~j \in 1..u\}$
the set of dyadic cubes in $\X^u,~u\in \N$ 
of side-length $2^{-v}$, and let  $\B^{(u)}:=\sigma(\{\Delta_{u,v},~v \in \N\})$  be the Borel subsets of $\X^u,~u \in \N$. Let $\X^{\N}$ be the set of all $\X$-valued infinite sequences equipped with the Borel $\sigma$-algebra
$\B:=\sigma(\{B\times \X^{\N}:B \in \Delta_{u,v},~u, v \in \N\})$. By a stochastic process we mean a probability measure on $(\X^{\N},\B)$. 
Take
a sequence of random variables $\x:=\langle X_t\rangle_{t \in \N}$ with joint distribution $\mu$ 
where for every $t \in \N$, $X_t:\X^\N \to \X$ is the coordinate projection of $\a:=\langle a_t \rangle_{t \in \N} \in\X^\N$ onto its $t^{\text{th}}$ element, i.e. $X_t(\a)=a_t$.
For each $n \in \N$ and $B \in \B^{(u)},~u \in \N$ we let 
$\mu_n(\x,B): \X^\N  \rightarrow [0,1],~n \in \N$ denote the empirical measure of $B$ where,
\begin{equation}\label{eq:mun}
\mu_n(\x,B) = \begin{cases}
\frac{1}{n-u+1}\sum_{i=1}^{n-u+1} \mathbb{I}\{X_{i..i+u}\in B\} &~n \geq u\\
0 & \text{otherwise}
\end{cases}
\end{equation}
and $\mathbb{I}$ is the indicator function.
We often call $X_{1..n}$ a sample of length $n \in \N$ generated by a stochastic process $\mu$ with corresponding sequence of random variables $\langle X_t \rangle_{t \in\N}$. 
\begin{definition}[Stationary Ergodic Processes] 
A process $\mu$ is stationary if
\begin{equation*}
\mu(X_{1..u}\in B)= \mu(X_{{1+j}..{u+j}} \in B)
\end{equation*} for every $B \in \B^{(u)},~u \in \N$ and every $j \in \N$.
A stationary process $\mu$ with corresponding sequence of random variables $\x=\langle X_t \rangle_{t \in \N}$ is (stationary) ergodic if with for every $u \in \N$ and $B \in \B^{(u)}$ it holds that 
\begin{equation*}
    \lim_{n \rightarrow \infty} \mu_n(\x,B)=\mu(B),~\mu-\as.
\end{equation*}
\end{definition} 
\begin{remark}
 The above definition can be shown to be equivalent to the standard definition involving triviality of shift-invariant measurable sets; see, e.g., \cite{Gray:88}. \end{remark} 

\paragraph{Joint process distributions.} We may require to simultaneously consider multiple samples $X_{1..n},~n \in \N$ generated by different, possibly dependent stationary ergodic processes. To allow for this, we first define a distribution over a matrix of random variables, each row of which shall correspond to one of the
samples. Next, we obtain each process as the {\em marginal} distribution of the corresponding row of the matrix.
More specifically, we have the following formulation. 
For a fixed $m \in \N$, let $\rho$ 
be a measure on the space $(\X^{m \times \N}, \B^{\otimes m})$ where,
\begin{equation}\label{eq:botimes}
\B^{\otimes m}:= \sigma(\{B_1 \times \dots \times B_{m}: B_i \in \B,~i \in 1..m\}).
\end{equation}
Define the matrix of $\X$-valued random variables
\begin{equation}\label{eq:bfX}
    {\bf X} := 
    \begin{pmatrix}
    X_{1,1}&X_{1,2}&X_{1,3}&\dots\\
    \vdots & \dots &\ddots & \ddots\\
     X_{m,1}&X_{m,2}&\dots&\dots \\
    \end{pmatrix}
\end{equation}
where $X_{i,j}:\X^{\N \times \N} \rightarrow \X,~i,j \in \N$ are jointly distributed according to ${\rho}$, so that for $B \in \B^{\otimes m}$ we have $\Pr({\bf X} \in B) = \rho(B)$. For each $i \in 1..m$, let $\x_i:=\langle  X_{i,j} \rangle_{j \in \N}$ and define the projection map 
$\pi_i \mapsto \x_i$. The marginal distribution $\mu_i$ of $\x_i$ is then defined as the distribution induced by $\rho$ over the $i^\text{th}$ row, i.e. 
$\mu_i:=\rho \circ \pi_i^{-1}$.
We denote by 
\begin{equation}\label{mrho}
\mathcal M(\rho):=\{\mu_i: ~i \in 1..m\}
\end{equation}
the set of marginal process distributions of $\rho$. 
\begin{definition}[Distributional Distance]\label{defn:dd}
A {\em distributional distance} between a pair of processes $\mu, \mu'$ is defined as 
\begin{equation*}
d(\mu,\mu'):=\sum_{u,v \in \N} w_u w_v \sum_{B\in \Delta_{u,v}} |\mu(B)-\mu'(B)|
\end{equation*}
where $w_j=1/j(j+1),~j \in \N$ or 
any summable sequence of positive weights.
\end{definition}
\begin{definition}[Empirical Estimates of Distributional Distance]\label{defn:edd}
Consider a pair of process marginals $\mu, \mu' \in \M(\rho)$ given by \eqref{mrho} with corresponding sequence of random variables $\x$ and $\x'$ respectively, where 
$\mu:=\mu_i,~\mu':=\mu_j$, and $\x:=\langle X_{i,t} \rangle_{t \in \N},~\x':=\langle X_{j,t}\rangle_{t \in \N}$ correspond to rows  $i,j \in 1..m$  of $\bf{X}$ given by \eqref{eq:bfX}.
Empirical estimates of $d(\mu,\mu')$ are given by  
\begin{align}
&\widehat{d}_n(\x,\x'):=\sum_{u,v \in \N}w_{u}w_{v}\sum_{B \in \Delta_{u,v}}|\mu_n(\x,B)-\mu_{n}(\x',B)|\label{eq:dhat},\\
& \widehat{d}_n(\x,\mu):=\sum_{u,v \in \N}w_{u}w_{v}\sum_{B \in \Delta_{u,v}}|\mu_n(\x,B)-\mu(B)|\label{eq:dhat2}
\end{align}
where $\mu_n(\cdot,\cdot)$ is given by \eqref{eq:mun} and $w_j,~j \in \N$ is as in Definition~\ref{defn:dd}.
\end{definition}
\begin{remark}\label{rem1}
Note that \eqref{eq:dhat} can be efficiently calculated with computational complexity
$\O(n\polylog n)$ for $u_n:=\log n,~v_n:=-\log(s_{\min})$, where
$s_{\min}$ is the minimal non-zero difference between the union of all the elements of the two sequences $\x,\x'$,  see \cite{khaleghi2016consistent}.
\end{remark}
\begin{proposition}[\cite{khaleghi2016consistent}]
If the marginals in $\M(\rho)$ given by \eqref{mrho} are stationary ergodic, then for any $\mu \in \M(\rho)$ and any $s,~t \in 1..m$ it holds that,
\begin{align*}
&\lim_{n \rightarrow \infty}\widehat{d}_n(\x_s,\mu) = d(\mu_s,\mu)~\rho-\as,\\
&\lim_{n \rightarrow \infty}\widehat{d}_n(\x_s,\x_t) = d(\mu_s,\mu_t)~\rho-\as
\end{align*}
where  $\x_j:=\langle X_{j,t} \rangle_{t \in \N}$ correspond to $j^{\text{th}}$ row of $\bf{X}$ given by \eqref{eq:bfX} and $\mu_j \in \M(\rho)$,~$j=s,~t$. 
\end{proposition}
\section{Problem formulation}\label{sec:prob}
\paragraph{Piecewise stationary ergodic processes.} 
In this paper, we shall be dealing with multiple samples of the form 
\begin{equation}\label{eq:pwss}
Y_1,\dots,Y_{\tau_1},Y_{\tau_1+1},\dots,Y_{\tau_2},\dots,\dots,Y_n,
\end{equation}
where the (stationary) segments $Y_{\tau_i},\dots,Y_{\tau_{i+1}}$ are generated by different, possibly dependent, stationary ergodic processes. To specify the distribution of the sample $\y$ we define a distribution over a matrix of random variables~\eqref{eq:bfX}, each row of which shall correspond to a stationary segment of the sample. The {\em stationary-segment distributions} of \eqref{eq:pwss} are then obtained by projecting the stationary segments onto the corresponding rows of the matrix. 
More formally, we specify a {\em Piecewise Stationary Ergodic Process} as follows. 
Consider a measure $\rho$ on $(\X^{\kappa \times \N},\B^{\otimes {\kappa+1}})$ for some fixed $\kappa \in \N$ with set of marginals $\M(\rho)=\{\mu_i,~i \in 1..\kappa+1\}$ as specified by \eqref{mrho}, 
where  $\mu_i \neq \mu_{i+1},~i \in 1..\kappa$ are stationary ergodic.
Fix some $n \in \N$ and a sequence $\boldsymbol{\tau}:= \langle \tau_i \rangle_{i \in 1..\kappa}$ with $\tau_1<\tau_2< \dots <\tau_{\kappa}\in 1..n$. 
Define the mapping $c: \N \rightarrow \N \times \N$ as 
\begin{equation}\label{defn:cmap}
c(j)\mapsto (t^*(j)+1,j-\tau_{t^*(j)})
\end{equation} 
where  $t^*(j):=\max_{i \in 0..\kappa+1} \tau_i \leq j$ picks out the closest change point $\tau_i$ to $j \in \N$, with the convention that $\tau_0:=0$ and $\tau_{\kappa+1}:=n$.
A {\em Piecewise Stationary Ergodic Sample} of the form \eqref{eq:pwss} generated by $(\rho, \boldsymbol{\tau})$ can be specified as a sequence of coordinate projections $Y_t: \X^{n} \rightarrow \X,~t \in 1..n$ such that for any $\ell \in 1..n,~t_1, \dots, t_\ell \in 1..n$ and $B_i \in \B_{\X},~i\in 1..\ell$ it holds that
\begin{equation}\label{eq:marginalmu}
    \Pr(Y_{t_1} \in B_1, \dots,Y_{t_\ell} \in B_\ell )= \rho(X_{c(t_1)} \in B_1, \dots,X_{c(t_\ell)} \in B_\ell).
\end{equation}
Observe that by \eqref{eq:marginalmu} the distribution of each segment $Y_{\tau_i+1..\tau_{i+1}}$ is given by a stationary ergodic process $\mu_{i},~i \in 1..\kappa+1$. Since it is assumed that $\mu_i \neq \mu_{i+1},~i \in 1..\kappa$, the indices $ \tau_i,~i \in 1..\kappa$ are referred to as {\em change points}. 
The pair $(\rho,  \boldsymbol{\tau})$ composed of the measure $\rho$ and its corresponding sequence of change points $ \boldsymbol{\tau}$ defines  
 a piecewise stationary ergodic process. 
\begin{definition}[Equivalence of piecewise stationary ergodic processes]\label{defn:equiv}
We consider a pair of piecewise stationary ergodic processes $(\rho,\boldsymbol{\tau})$ and $(\rho',\boldsymbol{\tau}')$ equivalent, if and only if they agree on their set of stationary-segment distributions, i.e., 
\begin{equation}\label{eq:equiv}
    (\rho,\boldsymbol{\tau}) \sim (\rho',\boldsymbol{\tau}') \Leftrightarrow \M({\rho}) = \M({\rho'}).
\end{equation} 
\end{definition}
Let $\mathcal P$ denote the set of all piecewise stationary ergodic processes. The equivalence relation defined above induces a partitioning of $\mathcal P$ into 
distinct classes $[(\rho,\boldsymbol{\tau})],~(\rho,\boldsymbol{\tau}) \in \mathcal P$
where 
\begin{equation*}
    [(\rho,\boldsymbol{\tau})]:=\big\{(\rho',\boldsymbol{\tau}') \in \mathcal P:(\rho',\boldsymbol{\tau}') \sim (\rho,\boldsymbol{\tau}) \big \}
\end{equation*}
so that two piecewise stationary ergodic processes belong to the same class, if and only if they are equivalent in the sense of \eqref{eq:equiv}.
We let 
\begin{equation}\label{eq:C}
    \mathcal C:=\{[(\rho,\boldsymbol{\tau})]:(\rho,\boldsymbol{\tau}) \in \mathcal P\}
\end{equation} denote the set of all such equivalence classes.
\paragraph{Clustering Problem.} 
Fix some $N \in \N$, which is the number of samples, and (unknown) sequence $\kappa_i \in \N,~i \in 1..N$ corresponding to the number of change points in each sample. Moreover, define (unknown) sequences 
$\boldsymbol{\theta}_i:=\langle \theta^{(i)}_j \rangle_{j \in 1..\kappa_i+1},~i \in 1..N$ with
$\theta^{(i)}_j \in (0,1)$.
For any $n \in \N$, define the change points $\boldsymbol{\tau}_i(n):=\langle  \tau^{i}_j(n)\rangle_{j \in 1..\kappa_i+1},~i \in \N$ where $\tau^{i}_j(n):=\lfloor n\theta_j^{(i)}\rfloor-\lfloor n\theta_{j-1}^{(i)}\rfloor $, with the convention that $\theta_{0}^{(i)} = 0$ for all $i  \in 1..N$. 
The problem is formulated as follows. For a fixed $n \in \N$, we are given a set 
\begin{equation}\label{eq:Nsamples}
\T(n):=\{\y_1,\dots,\y_N\}
\end{equation}
of $N$ piecewise stationary ergodic samples 
of the form \eqref{eq:pwss},  each generated by an unknown piecewise stationary ergodic process $(\rho_i,\boldsymbol{\tau}_i(n)),~i \in 1..N$. Thus, each sample is of length $n_i:=\lfloor n\theta_{\kappa_i+1}\rfloor$ and has $\kappa_i$ change points. It is assumed that each of $N$ piecewise stationary ergodic processes that generate the samples belongs to one of
$m$ distinct classes $C_1, \dots, C_m \in \C$, which are unknown.
Define the normalized minimum separation between the change points as 
\begin{equation}\label{eq:alpha}
\alpha:=\min_{i \in 1..N}\min_{k \in 1..\kappa_i+1}\theta_j^{(i)}-\theta_{j-1}^{(i)}. 
\end{equation}
We assume that $\alpha >0$ so that each stationary segment in $\y_i,~i \in 1..N$ is of length at least $n\alpha$.

\begin{definition}[Ground Truth Clustering]
Let $\G:=\{\G_1, \dots, \G_m\}$ be a partitioning of $1..N$  
where for any $i \in 1..N$, it holds that $i \in \G_{\ell}$ for some $\ell\in 1..m$ if and only if $(\rho_i,\boldsymbol{\tau}_i(n)) \in C_{\ell}$. We call $\G$ the ground-truth clustering.
\end{definition}
\noindent Thus, in the ground-truth clustering two samples fall into the same cluster if and only if the piecewise-stationary distributions that generate them are equivalent, in the sense that they have the same set of stationary distributions of the segments.

A clustering function $f$ takes a set $\T$ of samples along with the number $m$ of target clusters to produce a partition $f(\T,m)\mapsto\{J_1,\dots,J_m\}$ of $1..N$, aiming to recover the ground-truth $\G$. 
\begin{definition}[Consistency]
A clustering function $f$ is consistent for a
set of samples $\T=\T(n),~n \in \N$ if $f(\T, m) = \G$. Moreover, $f$ is called
asymptotically consistent if with probability $1$ it holds that 
\begin{equation*}
  \lim_{n \rightarrow \infty} f(\T(n),m) = \G.
\end{equation*}
\end{definition}

\paragraph{Joint distribution of piecewise stationary samples.}
Observe that the problem requires us to simultaneously consider  multiple samples, each generated by a piecewise stationary ergodic process.
These samples can themselves be dependent. Formally, this is defined through the following construction. 
Consider the space $\mathcal Y:= \X^{\kappa_1 \times\N} \times \dots \times  \X^{\kappa_N \times\N}$.  Denote by $\F:= \B_1 \otimes  \dots \otimes \B_N$ the product $\sigma$-algebra on $\Y$
where $\B_i:=\sigma(\{B_1 \times \dots \times B_{\kappa_i}: B_j \in \B\}),~i \in 1..N$ 
is in turn the product $\sigma$-algebra on $ \X^{\kappa_i \times \N}$.
Let $P$ be a probability measure on $(\Y, \F)$.
Consider a sequence ${\bf Z}:=\langle {\bf Y}_i \rangle_{i \in 1..N}$ of infinite matrices of  
$\X$-valued random variables
$Y_{s,t}^{(i)}:\X^{\kappa_i \times \N} \rightarrow \X,~s \in 1..\kappa_i,~t \in \N,i \in 1..N$, which can be easily shown to be $\F$-measurable. 
Suppose that $P$ is the distribution of ${\bf Z}$ so that $\Pr({\bf Z} \in F) = P(F)$ for all $F \in \F$. For each $i \in 1..N$, define the projection map $\widetilde{\pi}_i \mapsto \langle  Y_{s,t}^{(i)} \rangle_{s\in 1..k_i,~t \in \N}$. Then  $\rho_i:=P \circ \widetilde{\pi}_i^{-1},~i \in 1..N$ is the measure of ${\bf Y}_i$.
Our main probabilistic statements  will be stated in terms of $P$. 
\paragraph{The role of $n$.} The asymptotic results in this paper are all with respect to $n$ approaching infinity. While the formulation allows for all of the samples, and, of course, all of the stationary segments, to be of different lengths, $n$ parametrizes  (via the unknown sequences ${\bf \theta}_i$) all these lengths. Thus, when $n$ goes to infinity, the lengths of all the samples and all the individual segments within the samples also approach infinity. Note, however, that the clustering protocol is not ``online'', and thus the algorithm proposed below does not deal with sequences that grow over time, but only considers a fixed set of sequences of fixed lengths. The asymptotic results are thus to be interpreted as stating that if all the sequences are long enough then the algorithm is correct.

\section{Main Results}
In this section we outline the main results.
We start by introducing a distance between equivalence classes of piecewise stationary distributions.
Next, we show that this distance can be consistently estimated based on piecewise stationary samples. The distance estimates will then be used to construct an asymptotically consistent clustering algorithm. The proofs as well as auxiliary technical results are provided in Section~\ref{sec:proof}.
\begin{definition}[Distance between piecewise stationary ergodic classes] \label{defn:dp}
Let $C, C' \in \C$ be two classes of piecewise stationary ergodic processes where $C=[(\rho, \boldsymbol{\tau})]$ and $C'=[(\rho', \boldsymbol{\tau}')]$. 
We define a distance between $C,~C'$ as follows.
$$
\dd(C,C')=\max_{\mu \in \M(\rho)} \min_{\mu' \in \M(\rho')} d(\mu,\mu') + \max_{\mu' \in \M(\rho')} \min_{\mu \in \M(\rho)} d(\mu',\mu)
$$
\end{definition}
\begin{proposition}\label{prop-delta}
The distance $\dd$ induces a metric on the set of piecewise stationary ergodic equivalence classes.
\end{proposition}
We show that the distance $\delta$ can be estimated consistently.
The proposed approach to estimate $\delta$ is outlined in Algorithm~\ref{as}. This algorithm in turn relies on a 
procedure, introduced in \cite{khaleghi:12mce} that, given a sample $\y$ of the form \eqref{eq:pwss}
with $\kappa$ change points generated by a piecewise stationary process 
outputs an exhaustive list
of $K \geq \kappa$ candidate change-point estimates. While the algorithm from~\cite{khaleghi:12mce} does not attempt to estimate the number of change points $\kappa$, among the $K$ candidate change points that it outputs there are $\kappa$ consistent estimates of the unknown change points. 
It is worth noting that~\cite{khaleghi:12mce} establishes a much stronger property, namely, it sorts the list in such a way 
that  its first $\kappa$ elements estimate the true change points of $\y$. 
However, we shall not require this feature here: it is enough to have a list of {\em arbitrary order} that includes a correct estimate for each change point; for simplicity, we assume that the $K$ candidate change-points are sorted left-to-right. 
More precisely we have the following definition. 

\begin{definition}[List-estimator \cite{khaleghi:12mce}]\label{defn:list} 
A list-estimator is a 
function $\mathcal L$ that, given a sample of length $n \in \N$ with $\kappa \in \N$ change points 
generated by a piecewise stationary process $(\rho, \boldsymbol{\tau})$ 
produces a list 
$\psi_1 (n)  < \dots < \psi_{K}(n) \in \{1,\dots, n\}^{K}$ of some 
$K \geq \kappa$ candidate estimates. 
A list-estimator is said to be consistent if with probability $1$
there exists some $\mathcal I:=\{i^*_1,\dots,i^*_{\kappa}\}\subseteq 1..K$ such that 
\begin{align*}
  \lim_{n\rightarrow \infty}\max_{k \in 1..\kappa}\frac{1}{n}|\psi_{i^*_k}(n)-\tau_k|=0.  
\end{align*}
\end{definition}
\begin{algorithm}[!ht]
\caption{Calculating an empirical estimate of $\dd$}
\label{as}
\begin{algorithmic}[1]
\State 
{\bf INPUT}: $\y \in \X^{n_1},~\y' \in \X^{n_2}$, $~\lambda\in (0,1)$
\State{\bf \em Obtain a sequence of candidate change-point parameters in $\y$ and $\y'$ respectively, using the list-estimator from \cite{khaleghi:12mce}}
\begin{align}
    \boldsymbol{\widehat{\tau}}\gets \mathcal L(\y,\lambda) \quad \text{and}& \quad \boldsymbol{\widehat{\tau}'}\gets \mathcal L(\y',\lambda)  \label{eq:ups} 
\end{align}
\State{\bf \em Generate sets $\S$ and $\S'$ of consecutive stationary-segments corresponding to $\y$ and $\y'$}
\begin{align}
    \S &\gets \left  \{\overline{\y}_i:=\y_{\psi_{i-1}..\psi_i},~ i \in 1..|\boldsymbol{\widehat{\tau}}|+1: \langle \psi_i \rangle_{i \in 1..|\boldsymbol{\widehat{\tau}}|} =  \boldsymbol{\widehat{\tau}},~\psi_{0}:=1,~\psi_{|\boldsymbol{\widehat{\tau}}|+1}:=n_1 \right\}  \label{eq:S1} \\
    \S' &\gets \left  \{\overline{\y}'_i:=\y_{\psi'_{i-1}..\psi'_i},~i \in 1..|\boldsymbol{\widehat{\tau}'}|+1: \langle \psi'_i \rangle_{i \in 1..|\boldsymbol{\widehat{\tau}'}|} =  \boldsymbol{\widehat{\tau}'},~\psi'_{0}:=1,~\psi'_{|\boldsymbol{\widehat{\tau}'}|+1}:=n_2 \right\}   \label{eq:S'1}
\end{align}
\State {\bf \em Calculate an empirical estimate of the distance between the underlying distributions}
\begin{align*}
n &\gets \min\{\lambda n_1, \lambda n_2\}\\
\delta(\y,\y', \lambda) &\gets \max_{\overline{\y} \in \S}\min_{\overline{\y}' \in \S'} \widehat{d}_n(\overline{\y},\overline{\y}') + \max_{\overline{\y}' \in \S' } \min_{\overline{\y} \in \S} \widehat{d}_n(\overline{\y}',\overline{\y})
\end{align*}
\State {\bf OUTPUT}: $\delta(\y,\y', \lambda)$
\end{algorithmic}
\end{algorithm}
\begin{algorithm}[!ht]
\caption{Clustering piecewise stationary samples}
\label{alg:clust}
\begin{algorithmic}[1]
\State 
\bf {INPUT}: sequences $\T:=\{\y_1, \cdots, \y_N\}$, number $m$ of target clusters, parameter $\lambda$ 
\State{\bf \em Initialize $m$ points as cluster-centres}
\State{$c_1 \gets 1$}
\State{$C_1\gets\{c_1\}$}
\For{$\ell=2..m$}
\State $c_\ell\gets \min \{\argmax_{i=1..N} \displaystyle \min_{j=1..l-1} \delta(\y_i,\y_{c_{j}},\lambda)\}$, {\em where $\delta$ is given by Algorithm~\ref{as}} 
\State $C_\ell \gets \{c_\ell\}$ 
\EndFor
\State\bf{\em Assign the remaining points to appropriate clusters:}
\For {$i=1 .. N$}
\State $k \gets \argmin_{j\in \bigcup_{\ell=1}^m C_{\ell}}\delta(\y_i,\y_{j},\lambda)$
\State $C_\ell \gets C_\ell \cup \{i\}$
\EndFor
\State \bf {OUTPUT}: clusters $C_1, C_2, \cdots, C_{m}$
\end{algorithmic}
\end{algorithm}
\begin{proposition}[$\dd$ can be estimated consistently]\label{prop:const}
Consider the samples $\y,\y'$ generated by a distribution $P$ with piecewise stationary ergodic marginals $(\rho,\boldsymbol{\tau})$ and $(\rho',\boldsymbol{\tau'})$; the lengths of the samples are parametrized by $n$. 
Let  the estimate $\widehat{\dd}_n( \y,\y'):=\delta(\y,\y',\lambda)$ be obtained as the output of Algorithm~\ref{as} with $\y,~\y'$ and any $\lambda \in (0,\alpha]$ as input, where $\alpha$ is given by \eqref{eq:alpha}. Then
\begin{equation*}
    \lim_{n \rightarrow \infty}\widehat{\dd}_n( \y,\y') = \delta(C,C'),\quad P-\as
\end{equation*}
where, $C:=[(\rho,\boldsymbol{\tau})]$ and $C':= [(\rho,\boldsymbol{\tau}')]$ are the equivalence classes containing the piecewise stationary ergodic processes that generate $\y$ and $\y'$ respectively. 
\end{proposition}
Finally, the distance estimates given by Algorithm~\ref{as} can be used to construct a clustering algorithm.
The clustering algorithm, i.e. Algorithm~\ref{alg:clust}, starts by initializing the clusters using farthest-point initialisation, and then assigns the remaining samples to the nearest cluster. This is the same procedure as used in \cite{Ryabko:10clust,khaleghi2016consistent} to cluster stationary samples, but using a different distance, $\widehat{\dd}_n$. The initialisation procedure is due to \cite{Katsavounidis:94}.
\begin{theorem}[Algorithm~\ref{alg:clust} is asymptotically consistent]\label{thm:const}
Let $f(\T(n),m)=\algo\ref{alg:clust}(\T(n),m, \lambda)$ be the output of Algorithm~\ref{alg:clust} when provided with the set $\T(n)$ of piecewise stationary ergodic samples \eqref{eq:Nsamples}, along with the correct number $m$ of target clusters and some $\lambda \in (0,\alpha]$, where $\alpha$ is given by \eqref{eq:alpha}. It holds that,
\begin{equation*}
    \lim_{n \rightarrow \infty} f(\T(n),m) = \G, \quad P-\as
\end{equation*}
Moreover, the computational complexity of the algorithm is $\mathcal O(m N(n^2\polylog n + \lambda^{-2}n\polylog n))$. 
\end{theorem}
\section{Proofs}\label{sec:proof}
In this section we prove our main results. 
We start by providing a simple argument showing that $\delta$ is indeed a metric on $\C$. Next, we prove Proposition~\ref{prop:const}. To this end, 
we first introduce a notation via Definition~\ref{defn:connected}
to identify the stationary-segment distribution which generates the {\em largest connected portion} of a sub-sequence of a piecewise stationary ergodic sample. The proof of the proposition relies on some technical results, namely, on \cite[Theorem~1]{khaleghi:12mce} and \cite[Lemma~2]{khaleghi14} which are stated for the sake completeness, as well as on Lemma~\ref{lem:main}, stated and shown below. 
\begin{proof}[Proof of Proposition~\ref{prop-delta}]
It is easy to verify that, since  $d(\cdot,\cdot)$ is a metric,   
$\dd$ is positive, symmetric and satisfies the triangle inequality. It remains to show that $\dd(C,C') = 0$ if and only if $C=C'$ for all $C,~C' \in \C$. Consider two piecewise stationary ergodic processes $(\rho, \boldsymbol{\tau})$ and  $(\rho', \boldsymbol{\tau}')$ in $\C$. 
Let $\M(\rho):=\{\mu_i:i \in 1..k\}$ and $\M(\rho'):=\{\mu'_i: i \in 1..k'\}$ denote the sets of distinct marginals corresponding to $\rho$ and $\rho'$ respectively. 
If $\dd(\C,\C')=0$ it must hold that 
$\max_{i \in 1..k} \min_{j \in 1..k'} d(\mu_i,\mu'_j) = 0$ and $\max_{i=1 ..k'} \min_{j \in 1..k} d(\mu'_i,\mu_j)=0 $, leading to 
$\M(\rho) \subseteq \M(\rho')$ and $\M(\rho') \subseteq \M_{\rho}$ respectively, so that $\M(\rho) = \M(\rho')$. 
On the other hand, suppose that $\M(\rho) = \M(\rho')$. If $\mu \in \M(\rho)$ then $\mu \in \M(\rho')$ and hence
$\min_{\mu' \in \M(\rho')}d(\mu,\mu') = 0$, leading to
$\max_{i \in 1..k} \min_{j \in 1..k'} d(\mu_i,\mu'_j) = 0$. Similarly we obtain $\max_{i\in 1..k'} \min_{j \in 1..k} d(\mu'_i,\mu_j)=0 $, so that $\dd(\rho,\rho')=0$, and the result follows. 
\end{proof}
\begin{definition}\label{defn:connected}
Consider a sample $\y:=\langle Y_t \rangle_{t \in 1..n},~n \in \N$ generated by a piecewise stationary ergodic process $(\rho,\boldsymbol{\tau})$.
For any $u<v \in 1..n$ fix a sub-sequence $\overline{\y}:= \langle Y_t \rangle_{t \in u..v}$ of $\y$. 
We define $\mu^*(\overline{\y}):=\mu_{i^*}$ where 
\begin{equation}\label{eq:istar}
    i^*:=\min\argmax_{k \in 1..|\boldsymbol{\tau}|}\left |\{j \in u..v: (\max_{i \in 0..|\boldsymbol{{\tau}}|}\tau_i \leq j) = k\} \right|
\end{equation}
where the $\min$-operator is used to consistently break ties.
\end{definition}
\begin{theorem*}
There exists an asymptotically consistent list-estimator
$\mathcal L$ that, given a piecewise stationary ergodic sample of length $n \in \N$ with $\kappa \in \N$ change points 
generated by a piecewise stationary process $(\rho, \boldsymbol{\tau})$ 
produces a list 
$\psi_1(n),\dots,\psi_{K}(n) \in \{1,\dots, n\}^{K}$ of some 
$K \geq \kappa$ candidate estimates, that are at least $\lambda n$ apart. 
Let $\psi_{[1]}(n)\leq \psi_{[2]}(n) \leq \dots \leq \psi_{[\kappa]}(n)$
be the first $\kappa$ elements of $\psi_j,~j \in 1..K$ sorted in increasing order of value. 
With probability $1$ it holds that
\begin{align*}
\lim_{n\rightarrow \infty}\max_{j=1..\kappa}\frac{1}{n}|\psi_{[j]}(n)-\tau_j| =0.
\end{align*}
\end{theorem*}

\begin{lemma*}
Let $\y$ be a piecewise stationary ergodic sample of length $n$ with $\kappa$ change points and at least $\alpha n$ apart for some $\alpha \in (0,1)$. 
Let $\S$ be the set of segments of the form \eqref{eq:S1}. 
For all $\lambda \in (0,\alpha]$ with probability $1$ we have
\begin{equation*}
\lim_{n\rightarrow \infty} \max_{\overline{\y} \in \S}\widehat{d}_n(\overline{\y},\mu^*(\overline{\y})) = 0.
\end{equation*}
\end{lemma*}
\begin{lemma}\label{lem:main}
Consider the set $\T$ of piecewise stationary samples specified by \eqref{eq:Nsamples}.  
For any $j  \in 1..|\T|$ suppose that $\y_j \in \T$ along with some $ \lambda \in (0,\alpha]$ is provided as input to Algorithm~\ref{as}, where $\alpha$ given by \eqref{eq:alpha} denotes the minimum normalized separation between the change points. With probability $1$ it holds that, for every $\epsilon > 0$ there exists some $N$ such that for all $n \geq N$ we have
\begin{equation*}
     \M_n(\rho_j) = \M(\rho_j)
\end{equation*}
where $\M_n(\rho_j):=\{\mu^*(\overline{\y}): \overline{\y} \in \S\}$,~ $\S$ is specified by \eqref{eq:S1} in Algorithm~\ref{as} and $\mu^*(\cdot)$ is given by Definition~\ref{defn:connected}.
\end{lemma}
\begin{proof}
For simplicity of notation let $\y:=\y_j$,~$(\rho,\boldsymbol{\tau}(n)) = (\rho_j,\boldsymbol{\tau}_j(n)),~n \in \N$ and $\kappa:=\kappa_j$.
First observe that by definition
$\M_n(\rho) \subseteq \M(\rho)$ for all $n \in \N$.
Therefore, to prove the statement, it suffices to show that for large enough $n$ we have $\M(\rho) \subseteq \M_n(\rho)$.
Fix $\epsilon \in (0,\lambda/2)$. 
By \cite[Theorem~1]{khaleghi:12mce}, there exists some $N$ such that for all $n \geq N$ the list of candidate change-points produced by $\mathcal L$ is at least $\kappa$ long and includes a subset of size $\kappa$ composed of consistent change point estimates. That is, there exists some $\mathcal I:=\{i^*_1,\dots,i^*_{\kappa}\}\subseteq 1..|\boldsymbol{\widehat{\tau}}(n)|$ such that 
\begin{align}\label{lem1:const}
&\max_{k \in 1..\kappa}\frac{1}{n}|\psi_{i^*_k}-\tau_k|  \leq \epsilon\\
&\min_{i\in 1..|\boldsymbol{\widehat{\tau}}(n)|+1}\psi_i-\psi_{i-1}\geq n\lambda \label{prop0:eq:linlen}
\end{align}
where  $\psi_0:=0$ and  $\psi_{|\boldsymbol{\widehat{\tau}}(n)|+1}:=n$. Note that \eqref{prop0:eq:linlen} implies that the candidate estimates obtained by $\mathcal L$ are at least $n\lambda$ apart. 
We show that for every $\nu \in \M(\rho)$ there exists some $\overline{\y} \in\S$ such that 
$\nu=\mu^*(\overline{\y})$ so that $\nu \in \M_n(\rho)$ for all $n \geq N$. 
Take some $\nu \in \M(\rho)$. 
By construction, there exists some $k \in 0..\kappa$ such that 
$\y_{\tau_k+1..\tau_{k+1}}$ is generated by $\nu$. 
By \eqref{lem1:const} there exists some $i^*_k \in \mathcal I$ such that  
$\frac{1}{n}|\psi_{i^*_k}-\tau_k|  \leq \epsilon$. 
We show that $[\psi_{i^*_k},\psi_{i^*_k+1}] \subseteq [\tau_k-\epsilon n,\tau_{k+1}+\epsilon n]$. We have two cases: either $i^*_k+1 \in \mathcal I$, in which case by  \eqref{lem1:const} 
it holds that $\frac{1}{n}|\psi_{i^*_k+1}-\tau_{k+1}|  \leq \epsilon$ for all $n \geq N$, or 
$i^*_k+1 \in \mathcal I^c:=\{1,\dots, |\boldsymbol{\widehat{\tau}}(n)|\}\setminus \mathcal I$. 
We argue that in the latter case $\psi_{i^*_k+1} < \tau_{k+1}$. 
To see this, assume by way of contradiction that $\psi_{i^*_k+1} > \tau_{k+1}$
where $\tau_{k+1} \neq n$; note that the statement trivially holds for $\tau_{k+1}=n$. 
By the consistency of $\mathcal L(\y,\lambda)$ there exists some $j > i^*_k \in \mathcal I$
such that $$\frac{1}{n}|\psi_j-\tau_{k+1}|\leq \epsilon$$ for all $n \geq N$. 
Moreover, by \eqref{lem1:const} and \eqref{prop0:eq:linlen} for all $n \geq N$ 
the candidates indexed by the elements of $\I^c $ are linearly separated
from the true change points, that is,
\begin{align}\label{prop:eq:mingap}
&\min_{\substack{k\in1..\kappa\\ i\in \mathcal I^c}}|\tau_k-\psi_i|  \geq \min_{\substack{k\in1..\kappa \\ i\in \I^c,j\in \I}}|\psi_i-\psi_j|-|\tau_k-\psi_j| \geq n (\lambda-\epsilon).
\end{align}
Thus, from \eqref{lem1:const}  and \eqref{prop:eq:mingap} 
we obtain that $\psi_{i^*_k}-\psi_j \geq \lambda-2\epsilon >0$. 
Since by construction $\psi_i,~i \in 0..|\boldsymbol{\widehat{\tau}}(n)|+1$ are sorted in increasing order of value, this leads to a contradiction. Therefore, we have 
$[\psi_{i^*_k},\psi_{i^*_k+1}] \subseteq [\tau_k-\epsilon n,\tau_{k+1}+\epsilon n]$. 
It follows that $\mu^*(\overline{\y})=\nu$, where $\overline{\y} = \y_{\psi_{i^*_k}..\psi_{i^*_k+1}}$. 
This implies that $\nu \in \M_n(\rho)$, and the statement follows. 
\end{proof}
\begin{proof}[Proof of Proposition~\ref{prop:const}]
Let
\begin{equation}\label{eq:deltamin}
    \delta_{\min}:=\min_{\substack{\mu \neq \mu' \\ \mu \in \M(\rho),~ \mu' \in \M(\rho')}}d(\mu,\mu')
\end{equation} 
denote the minimum non-zero distributional distance between the stationary-segment distributions in $\M(\rho)$ and $\M(\rho')$. Fix some $\epsilon \in (0,\delta_{\min}/3)$ . 
Let $\S$ and $\S'$, respectively specified by \eqref{eq:S1} and \eqref{eq:S'1}, be the stationary segments identified by Algorithm~\ref{as} in each sample. 
Note that by construction the maximum number of segments produced by $\mathcal L(\cdot,\lambda)$ is $1+\lambda^{-1}$ 
so that with $\lambda \leq \alpha$ we have 
\begin{equation}\label{eq:maxss}
\max\{|\S|,|\S'| \}\leq 1+\alpha^{-1}.
\end{equation}
where $\alpha$ given by \eqref{eq:alpha} is the minimum normalized distance between the change points. 
Moreover, note that $n_i = n\theta^{(i)}_{\kappa_{i+1}}$ and $n_j= n\theta^{(j)}_{\kappa_{j+1}}$ where $\theta_{\kappa_{i+1}}^{(i)},~\theta^{(j)}_{\kappa_{j+1}} \in (0,1)$, we have $\max\{n_i,n_j\} \leq n$ so that for any $\lambda \in (0,\alpha]$ it holds that 
\begin{equation}\label{eq:goodl}
    \max\{n_i,n_j\}\lambda \leq n \alpha \leq n\theta_{\min}
\end{equation}  
where $\theta_{\min}:=\min\{\theta_k^{(i)},~\theta_{k'}^{(j)},~k \in 1..\kappa_{i+1},~k' \in 1..\kappa'_{j+1}\}$ corresponds to the minimum length of the stationary segments in $\y$ and $\y'$.
By \eqref{eq:goodl} the conditions for \cite[Lemma~2]{khaleghi14} hold so that by this lemma there exists some $N_0 \in \N$ such that for all $n \geq N_0$ we have
\begin{align}\label{eq:je}
   \max\left \{\max_{\overline{\y} \in \S}\widehat{d}_n(\overline{\y},\mu^*(\overline{\y})), 
   \max_{\overline{\y}' \in \S'}\widehat{d}_n(\overline{\y}',\mu^*(\overline{\y}'))  \right \} \leq \epsilon
\end{align} 
Moreover, by Lemma~\ref{lem:main} there exists some $N_1$ such that for all $n \geq N_1$ it holds that 
\begin{align}\label{eq:mainlem}
    \M_n(\rho) = \M(\rho)\quad\text{and}\quad \M_n(\rho') = \M(\rho').
\end{align}
To prove the statement we proceed as follows. 
Take $n \geq \max \{N_0,N_1\}$. 
We have,
\begin{align}
\max_{\overline{\y} \in \S}\min_{\overline{\y}' \in \S'} \widehat{d}_n(\overline{\y},\overline{\y}') 
&\leq \max_{\overline{\y} \in \S}\min_{\overline{\y}' \in \S'} \widehat{d}_n(\overline{\y},\mu^*(\overline{\y}')) + \widehat{d}_n(\mu^*(\overline{\y}'),\overline{\y}')\label{eq:suf1}\\
&\leq \max_{\overline{\y} \in \S}\min_{\overline{\y}' \in \S'} d(\mu^*(\overline{\y}),\mu^*(\overline{\y}')) + \widehat{d}_n(\overline{\y},\mu^*(\overline{\y})) + \widehat{d}_n(\mu^*(\overline{\y}'),\overline{\y}')\label{eq:suf2}\\
& \leq  \max_{\overline{\y} \in \S}\min_{\overline{\y}' \in \S'} d(\mu^*(\overline{\y}),\mu^*(\overline{\y}'))  + 2\epsilon\label{eq:suf3}\\
& = \max_{\mu \in \M_n(\rho)} \min_{\mu' \in \M_n(\rho')}d(\mu,\mu')+2\epsilon\label{eq:suf4}\\
&=\max_{\mu \in \M(\rho)} \min_{\mu' \in \M(\rho')}d(\mu,\mu')+2\epsilon\label{eq:suf5}
\end{align}
where \eqref{eq:suf1} and \eqref{eq:suf2} follow from the application of triangle inequality, \eqref{eq:suf3} follows from \eqref{eq:je}, and \eqref{eq:suf4} follows  from the definition of $\M_n$ and \eqref{eq:suf5} follows from\eqref{eq:mainlem}. On the other hand, 
\begin{align}
\max_{\overline{\y} \in \S} \min_{\overline{\y}' \in \S'}\widehat{d}_n(\overline{\y}, \overline{\y}')
&\geq \max_{\overline{\y} \in \S} \min_{\overline{\y}' \in \S'}\widehat{d}_n(\overline{\y},\mu^*(\overline{\y}'))-\widehat{d}_n(\overline{\y}', \mu^*(\overline{\y}')) \label{eq:nec2}\\
&\geq \max_{\overline{\y} \in \S}  \min_{\overline{\y}' \in \S'}d(\mu^*(\overline{\y}'), \mu^*(\overline{\y}))-\widehat{d}_n(\overline{\y},\mu^*(\overline{\y}))-\widehat{d}_n(\overline{\y}', \mu^*(\overline{\y}'))  \label{eq:nec3}\\
& \geq \max_{\overline{\y}' \in \S'} \min_{\overline{\y} \in \S}d(\mu^*(\overline{\y}), \mu^*(\overline{\y}'))-2\epsilon \label{eq:nec4}\\
&=\max_{\mu \in \M_n(\rho)}\min_{\mu' \in \M_n(\rho')}d(\mu,\mu')-2\epsilon\label{eq:nec5}\\
&=\max_{\mu \in \M(\rho)}\min_{\mu' \in \M(\rho')}d(\mu,\mu')-2\epsilon \label{eq:nec6}
\end{align}
where \eqref{eq:nec2} and \eqref{eq:nec3} follow from the triangle inequality, \eqref{eq:nec4} follows from \eqref{eq:je}, \eqref{eq:nec5} follows from the definition of $\M_n$ and \eqref{eq:nec6} follows from \eqref{eq:mainlem}. 
By \eqref{eq:suf4} and \eqref{eq:nec6} be obtain 
\begin{equation}\label{eq:first}
\left |\max_{\overline{\y} \in \S} \min_{\overline{\y}' \in \S'}\widehat{d}_n(\overline{\y}, \overline{\y}') - \max_{\mu \in \M(\rho)}\min_{\mu' \in \M(\rho')}d(\mu,\mu')\right | \leq 2\epsilon.
\end{equation}
Similarly, we have 
\begin{equation}\label{eq:second}
\left |\max_{\overline{\y}' \in \S'} \min_{\overline{\y} \in \S}\widehat{d}_n(\overline{\y}', \overline{\y}) - \max_{\mu' \in \M(\rho')}\min_{\mu \in \M(\rho)}d(\mu',\mu)\right | \leq 2\epsilon.
\end{equation}

\noindent From \eqref{eq:first} and \eqref{eq:second} it follows that 
$|\delta(\y,\y', \lambda)-\delta(C,C')| \leq 4\epsilon$. 
Since the choice of $\epsilon$ is arbitrary, the statement follows. 
\end{proof}
\begin{proof}[Proof of Theorem~\ref{thm:const}]
The proof can be recovered from that of \cite[Theorem~1]{khaleghi2016consistent} with the distributional distance $d$ replaced by $\dd$, since the algorithms only differ in this choice of distance. The consistency of the algorithm then follows from the consistency of $\widehat{\dd}_n$, namely, Proposition~\ref{prop:const}. 
For the sake of completeness, let us recall the main argument from the proof of \cite[Theorem~1]{khaleghi2016consistent}.
The ground-truth clustering $\mathcal G$ has the so-called strict separation property with respect to the distance $\dd$, meaning that any two distributions in the same cluster are closer to each other than to those in any other cluster: the $\dd$ distance is 0 if and only if the distributions are the same and thus in the same cluster. From the consistency of $\widehat{\dd}_n$ it follows that the same holds $P$-\as from some $n$ on for the $\widehat{\dd}_n$ distance. 
To  calculate the computational complexity of the algorithm,  first note that by \cite[Remark~2]{khaleghi:12mce} the computational complexity of the list-estimator is $\mathcal O(n^2\polylog n)$, and that of the distributional distance $\widehat{d}$ is $\mathcal O(n\polylog n)$. Moreover, to obtain an estimate of $\widehat\dd_n$, the algorithm needs two calls to the list estimator as well as an additional $1/\lambda^2$ (corresponding the maximal number of estimated segments in each sample) calculations of $\widehat d$. Given that the total number of calculations of $\widehat\dd_n$ of the algorithm is $mN$, this brings the overall computational complexity to $\mathcal O(m N(n^2\polylog n + \lambda^{-2}n\polylog n))$.
\end{proof}
\section{Conclusion}
In this paper, we have introduced a novel probabilistic framework for 
clustering time series, which is considerably more general than 
previously used models. At the same time, as shown here, it allows for provably 
consistent efficient algorithms.
In this section we analyse the 
conditions of the main theorem, and briefly outline some possible extensions and 
generalizations of the results obtained.
\paragraph{Necessity of the conditions.}
In Theorem~\ref{thm:const} it is required that the correct number of clusters 
$m$, as well as a lower-bound on the minimal distance 
between change points, be provided. The former requirement is necessary. Indeed, as 
shown in \cite{Ryabko:10discr} (see also \cite{Ryabko:19c}), for 
stationary ergodic distributions there is no asymptotically consistent 
algorithm that, given two samples, would answer whether they were 
generated by the same process or by different ones.
As a corollary, without knowing the number of clusters, it is, in 
general, impossible to cluster even two samples generated by 
(single-piece and thus, of course, also piecewise) stationary 
distributions. As for the latter requirement, the lower-bound $\lambda$ 
on the minimal distance between change points, comes from the 
corresponding requirement of the change-point estimation algorithm of 
\cite{khaleghi14}. Currently, it is not known whether this requirement 
is necessary, although we would conjecture that it is.  In contrast, if 
the number of change points  is known, then these change points can be 
estimated consistently without $\lambda \in (0,\alpha]$ given, as is 
established in \cite{Khaleghi:15chp}. While of lesser practical 
interest, an analogous consistency result can be established for the 
clustering problem considered here. To do so, Algorithm~\ref{as} would 
obtain the list of change point estimates in each sample by applying 
the algorithm of  \cite{Khaleghi:15chp}, and then proceed without 
further modifications. Thus, the requirement of a known lower-bound on the 
minimum distance between change points can be traded for the requirement 
of a known number of change points in each sample.
\paragraph{Extensions and generalizations.}
The time series studied in this paper are assumed to be generated by 
finitely many piecewise stationary distributions, with changes in 
distribution being abrupt. One can first ask whether stationary 
distributions can be generalized; a related question is whether the 
transitions can be made gradual. Both questions are answered by 
considering so-called {\em asymptotic mean stationary} distributions. Essentially, 
these are processes such that all frequencies converge almost surely; 
see \cite{Gray:88} for a formal definition and results. Here, it is worth 
mentioning that, since all that we use is asymptotic convergence of 
frequencies, all the results can be directly extended to this case; this 
includes the corresponding results on change-point estimation, see 
\cite{Ryabko:19c}. What this means is that the changes between 
distributions do not have to be abrupt, as long as they happen over 
$o(n)$ time-steps.
A different generalization could be achieved by considering the online setting of the 
problem. This means allowing the samples to grow with time, as well as 
adding new samples potentially at every time step. Such a setting allows 
one to accommodate a range of new applications that deal with growing 
bodies of data. The corresponding problem for stationary time series is 
addressed in \cite{khaleghi2016consistent}. Piecewise stationary 
samples present new challenges in this respect. Specifically, we would 
need to allow for an infinite number of change points. More importantly, 
it would become possible that two samples are generated by equivalent 
piecewise stationary distributions, but not all of the distributions of 
the samples have been revealed in each of the segments, and so the 
distributions look differently. A detailed investigation of what is 
possible to achieve in this setting is left for future work.
\paragraph{Finite-time guarantees.} In the framework of stationary 
(ergodic) time series, fundamental results establish the impossibility 
of obtaining any finite-time guarantees on the error of the resulting 
algorithms; indeed, already the speed of convergence of frequencies may 
be arbitrary slow \cite{Shields:96}. Therefore, additional assumptions that 
go beyond stationarity and ergodicity are necessary  if one wishes to  
obtain any finite-time guarantees on the performance. While this falls 
out of scope of the present paper, it would not be without interest to 
consider clustering piecewise stationary mixing or even i.i.d.\ time 
series: to the best of our knowledge, these questions would still be 
open. Making such assumptions would also lay grounds for the possibility of 
constructing consistent algorithms that do not require the knowledge of 
the correct number of clusters. The stationary mixing case of this 
problem is briefly considered in \cite{khaleghi2016consistent}.

\end{document}